\documentclass[nohyperref]{article}

\usepackage{microtype}
\usepackage{graphicx}
\usepackage{subfigure}
\usepackage{booktabs} 

\usepackage[colorlinks=true, allcolors=blue]{hyperref}

\usepackage[archive]{icml2022}

\usepackage{amsmath}
\usepackage{amssymb}
\usepackage{mathtools}
\usepackage{amsthm}

\usepackage[capitalize,noabbrev]{cleveref}

\usepackage[textsize=tiny]{todonotes}

\usepackage[utf8]{inputenc} 
\usepackage[T1]{fontenc}    
\usepackage{url}            
\usepackage{booktabs}       
\usepackage{amsfonts}       
\usepackage{nicefrac}       
\usepackage{microtype}      

\usepackage[english]{babel}

\usepackage{enumitem}
\usepackage{graphicx}
\usepackage{bm}

\usepackage{xcolor}

\usepackage[title]{appendix}
\usepackage{comment}
\usepackage{amsfonts}
\usepackage{dsfont}
\usepackage{lipsum}
\usepackage{dsfont,float}

\usepackage{algorithm}
\usepackage[labelfont=bf, font=normalsize]{caption}

\theoremstyle{plain}
\newtheorem{theorem}{Theorem}[section]
\newtheorem{condition}[theorem]{Condition}
\newtheorem{proposition}[theorem]{Proposition}
\newtheorem{lemma}[theorem]{Lemma}

\theoremstyle{definition}
\newtheorem{definition}[theorem]{Definition}
\newtheorem{assumption}[theorem]{Assumption}
\theoremstyle{remark}

\newcommand{\mS}{\mathcal{S}}	
\newcommand{\mA}{\mathcal{A}}	
\newcommand{\mB}{\mathcal{B}}	

\newcommand{\mI}{\mathcal{I}}	
\newcommand{\mD}{\mathcal{D}}	
\newcommand{\mR}{\mathcal{R}}	
\newcommand{\indic}{\mathds{1}} 

\newcommand{\PP}{\mathds{P}}    

\newcommand{\Eps}{\mathcal{E}}

\newcommand{\Exs}{\mathbb{E}}

\allowdisplaybreaks

\newif\ifdraft
\drafttrue  

\icmltitlerunning{Coordinated Attacks against Contextual Bandits}

\begin{document}

\twocolumn[
\icmltitle{Coordinated Attacks against Contextual Bandits: Fundamental Limits \\ and Defense Mechanisms}

%
%
%
%



\icmlsetsymbol{equal}{*}

\begin{icmlauthorlist}
\icmlauthor{Jeongyeol Kwon}{ut}
\icmlauthor{Yonathan Efroni}{msr}
\icmlauthor{Constantine Caramanis}{ut}
\icmlauthor{Shie Mannor}{tech}
\end{icmlauthorlist}

\icmlaffiliation{ut}{Department of Electrical and Computer Engineering, University of Texas at Austin, USA}
\icmlaffiliation{tech}{Technion, Israel}
\icmlaffiliation{msr}{Microsoft Research, New York}
\icmlcorrespondingauthor{Jeongyeol Kwon}{kwonchungli@utexas.edu}
\icmlcorrespondingauthor{Yonathan Efroni}{jonathan.efroni@gmail.com}

\icmlkeywords{Machine Learning, ICML}

\vskip 0.3in
]



\printAffiliationsAndNotice{}  

\begin{abstract}
    Motivated by online recommendation systems, we propose the problem of finding the optimal policy in multitask contextual bandits when a small fraction $\alpha < 1/2$ of tasks (users) are arbitrary and adversarial. The remaining fraction of good users share the same instance of contextual bandits with $S$ contexts and $A$ actions (items). Naturally, whether a user is good or adversarial is not known in advance. The goal is to robustly learn the policy that maximizes rewards for good users with as few user interactions as possible. 
    Without adversarial users, established results in collaborative filtering show that $O(1/\epsilon^2)$ per-user interactions suffice to learn a good policy, precisely because information can be shared across users. This parallelization gain is fundamentally altered by the presence of adversarial users: unless there are super-polynomial number of users, we show a lower bound of $\tilde{\Omega}(\min(S,A) \cdot \alpha^2 / \epsilon^2)$ {\it per-user} interactions to learn an $\epsilon$-optimal policy for the good users. 
    We then show we can achieve an $\tilde{O}(\min(S,A)\cdot \alpha/\epsilon^2)$ upper-bound, by employing efficient robust mean estimators for both uni-variate and high-dimensional random variables. We also show that this can be improved depending on the distributions of contexts.  
\end{abstract}

\section{Introduction}
Online recommendation systems \cite{adomavicius2005toward, li2010contextual} are ubiquitous, and used in diverse platform applications including video streaming, online shopping, travel and restaurant recommendations. 
These are large-scale systems, designed for millions of users and items. Thus, successful algorithms exploit shared preferences across different users (as in collaborative filtering). This naturally leads to a multitask contextual bandit framework \cite{sarwar2001item, maillard2014latent, sen2017contextual, chawla2020gossiping, yang2020impact, ghosh2021collaborative}. 
Manipulation of collaborative filtering recommendation engines is well documented, and various approaches to building resilient algorithms have been considered e.g., \cite{olsen2002amazon,van2010manipulation,chen2015matrix}. The ability to learn across users changes dramatically even in the presence of a small fraction of adversarial users, thus posing fundamental new challenges. This is precisely the problem we address.

\paragraph{Problem Setup.} We consider a multitask contextual bandit framework where the majority of users share preferences over all contexts and items. We call these the ``good'' users. Some $\alpha$-fraction of users ($\alpha < 1/2$), however, may have arbitrary preferences or may even attempt to manipulate the recommendation system; we call these the ``adversarial'' users. 

Formally, our system consists of $L$ users, a set of contexts $\mS$, and a set of actions $\mA$. Contexts are the temporal status of the user such as browsing histories, keywords, categories, etc.\ and the key assumption is that the good users share preferences across contexts, thus enabling cross-user learning (collaborative filtering). We note that, in particular, the user ID is a unique identifier, and not part of the context (in contrast to, e.g., \cite{sen2017contextual, deshmukh2017multi}). 
At each step, a user $i \in [L]$ arrives and reveals its current status (context) $s \in \mS$. Then, based on the context, we select an action $a \in \mA$ and get a reward (feedback) from the user. If the user is good, then the mean-reward returned by this user is $\mu(s,a) \in \mathbb{R}$. If the user is adversarial, it can return any reward based on all interaction histories. 

\paragraph{Goal of the Paper.} Our goal is to find a provably-approximately-correct (PAC) policy for the majority of good users, with the minimum per-user interactions. 
Without adversarial users, {\it i.e.,} assuming all users are share preferences across items and contexts, the problem is the well-studied contextual bandit problem \cite{lu2010contextual, chu2011contextual}. Classical results show that as soon as we have enough users, $L \ge SA$, the common preferences can be exploited, and existing contextual bandit algorithms  \cite{auer2002finite, zhou2015survey} can find a good policy after $O(SA/L)\sim O(1)$ per-user interactions. That is, we can achieve $O(1/L)$ {\it parallelization gain}, the ratio between required per-user interactions with a single user ($L = 1$) and many users. 

However, exploiting common preferences to reduce per-user interactions makes us susceptible to manipulation: a small but constant fraction $\alpha = O(1)$ of adversarial users can easily manipulate the state of the art contextual bandit algorithms. Moreover, as we show, a full parallelization gain of $O(1/L)$-factor is not possible even with a single context, for any algorithm if $\alpha > 0$. On the other hand, we can always  abandon parallelization gain, and learn an optimal policy for each user, with $O(SA/\epsilon^2)$ per-user interactions, thus learning is indeed possible. Our goal is to investigate the extent to which we can beat this naive baseline, while maintaining robustness in the presence of adversarial users.

\paragraph{Main Results.} We show that using recent ideas from robust statistics, we can design a robust contextual bandit algorithm that requires $O(\min(S,A) \cdot \alpha / \epsilon^2)$ per-user interactions, i.e., it achieves a $O(\alpha/\max(S,A))$ parallelization gain. As this falls short of the $L\cdot A$ parallelization gain in the absence of manipulators, a fundamental question is whether this is the best possible we can achieve. We show that the scaling in $A$ and $S$ cannot be improved, and we give a lower bound that matches up to one factor of $\alpha$. We summarize our results:
\begin{itemize}
    \item {\bf Robust Algorithm:} We initially consider two regimes: $A \gg S$, and $S \gg A$. In the first setting, we show that partitioning the users into $A$ groups, and within each group playing the same arm $1/\epsilon^2$ times. Then using median estimator, we show we can obtain $\epsilon$-accurate estimates of each arm in each context, thus giving a $O(1/A)$ parallelization gain. When $S \gg A$, we can do better than $(1/A)$ gain. Yet we need a different idea, since we may not see the same user--context pair enough times. Instead, we show we can leverage recent results in high dimensional mean estimation (e.g., \cite{diakonikolas2019robust}), to get a $O(1/S)$ parallelization gain.
    \item {\bf Fundamental Limits:} Can the $\max(S,A)$ parallelization gain be further improved once sufficiently many users $L = \Omega(SA \log (SA))$ are given? We show that, perhaps surprisingly, $\tilde{\Omega}(\min(S,A) \cdot \alpha^2/\epsilon^2)$ per-user sample complexity cannot be improved unless super-polynomial number of users are available. This is the first kind of negative results of learning contextual bandits in multitask settings (Section \ref{section:lower_bound}). 
    \item {\bf Problem Dependent Results:} In more practical scenarios, a certain set of contexts might be more frequently observed. For instance, there can be certain genres, keywords or categories that are more popular and frequently searched by users, while less popular contexts are rarely searched and thus less critical for the overall performance. We show that we can unify efficient uni-variate and high-dimensional robust estimators, and obtain the improved problem-dependent per-user sample-complexity (Section \ref{section:main_algorithm}). 
\end{itemize}

\subsection{Comparison to Previous Work}
Multitask contextual bandits have attracted significant recent attention, for a variety of applications. We only review the most closely related theoretical results on this topic. 

\paragraph{Multitask Contextual Bandits.} Multitask contextual bandits have been considered largely in two problem settings: (i) tasks of all user preferences are assumed to be embedded in low dimensional spaces \cite{sen2017contextual, gopalan2016low, yang2020impact, hu2021near}, (ii) users can be clustered into a small number of groups (compared to the number of users) that share the same instance of contextual bandits \cite{maillard2014latent, gentile2014online, gentile2017context, ghosh2021collaborative}. In the former setting (i), low-dimensional representations of tasks are the key objective to recover, and this has been done with general-purpose techniques such as tensor-decomposition or low-rank factorization. 
The adversarial users destroy this low dimensionality, hence these methods are not directly applicable to our setting. The latter setting (ii), parallelization gain leverages preference similarity within clusters \cite{maillard2014latent, gentile2014online, gentile2017context, ghosh2021collaborative}. While our problem is more closely related to this setting, again the challenge comes from the adversarial users. As their behavior can be arbitrary, we thus may have $O(L)$ unbalanced clusters among $L$ users, thus destroying our ability to leverage clustering. And indeed, the referenced work requires a small (in fact, a constant) number of clusters. We note that even in our work, we never identify the good vs.\ adversarial users.

\paragraph{Corruption Robust Bandits.} Corruption robust bandit algorithms \cite{gupta2019better, lykouris2018stochastic, lykouris2021corruption, ma2021adversarial, liu2021cooperative} consider the setting where the reward feedback can be corrupted at any time by an adversary with limited budget. There is little that can be done when the corruption budget scales linearly with the number of interactions. If we cast our problem in this setting, then though the corruption budget scales linearly, our adversary is limited and cannot corrupt the rewards of the good users. This additional information eventually allows us to learn the best policy for the underlying contextual bandit even with large total amount of corruptions. Necessarily, the strategies we develop are fundamentally different than those in the works references above.

\paragraph{Exploration in Contextual Bandits.} There is a long line of work that studies efficient exploration algorithms for contextual bandits (and reinforcement learning) in both stochastic and adversarial reward settings \cite{sutton2018reinforcement, auer2002finite, abbasi2011improved, even2006action, lattimore2020bandit, audibert2009minimax, bubeck2012best, gerchinovitz2016refined}. In the adversary-free setting, {\it i.e.,} $\alpha = 0$, we can use any of these algorithms, ignoring the user identifier. However, with $\alpha = O(1)$ fraction of adversarial users, algorithms that ignore user identifiers can be easily manipulated.

\paragraph{Other Related Work.} Contextual bandits have also been studied when rewards are represented as a linear combination of features of users and action items \cite{abbasi2011improved, chu2011contextual}. In this work, we do not assume that any prior information (such as the features of context-actions) is given in advance. Another line of work considers regret-minimization problem in latent bandits/MDPs where we can interact with each user for a fixed short time-horizon \cite{brunskill2013sample, kwon2021rl, zhou2021regime, kwon2021reinforcement}. In contrast, we allow as many interactions with each user as needed, while minimizing the number of per-user interactions.

\section{Preliminaries}
\label{section:prelim}

We consider a setting with $L$ users. Of these, a $(1-\alpha)$-fraction, called {\em good} users, have the same preferences, and hence follow a multitask contextual bandit framework $\mB := (\mS, \mA, \nu, \mR)$. There are $S$ contexts $\mS$ drawn according to distribution $\nu$, and $A$ actions $\mA$. The rewards follow the same distribution $\mR(s,a)$ under context $s$ and action $a$. We use $\mu(s,a)$ to denote the vector of mean rewards for the good users, and we assume that the reward distribution has bounded mean and variance: $|\mu(s,a)| \le 1$ and $\Exs_{r \sim \mR(s,a)}[(r - \mu(s,a))^2] \le 1$. The remaining $\alpha$-fraction of users are adversarial, and do not follow model $\mB$: their returns are not bound by any distribution, and can in fact be a function of the history to that point. 
Our goal is to find a policy $\pi \in \Pi: \mS \rightarrow \Delta(\mA)$ that maximizes the expected reward for the good users, {\it i.e.,} $\mB$:
\begin{align*}
    V_\mB^* = \max_{\pi \in \Pi} V_\mB^\pi := \Exs_{s \sim \nu} [\mu{(s,a)} | a \sim \pi(s)].
\end{align*}

Let $\mathcal{I}^* \subseteq [L]$ be the set of good users where $|\mI^*| \ge (1-\alpha) L$. Our interaction model is defined as follows: 
\begin{enumerate}
    \item At each step $t \in \mathbb{N}$, nature selects a user $i_t \in [L]$ according to some unknown process $\Gamma$ which call the user-arrival model. The only restriction on $\Gamma$ is that the difference in user frequency is controlled -- see Assumption \ref{assumption:user_selection} below.
    \item If the user is good ($i_t \in \mI^*$), then the user samples a context $s_t$ from $\nu$. Then, the algorithm selects an action $a_t$ based on $s_t$, and gets reward $r_t$ sampled from $\mR(s_t,a_t)$. 
    \item If the user is adversarial ($i_t \notin \mI^*$), then the adversary chooses an arbitrary context $s_t \in \mS$, and returns an arbitrary reward $r_t \in \mathbb{R}$.
\end{enumerate}

The adversarial users can choose any $s_t \in \mS$ and $r_t \in \mathbb{R}$ based on all histories, the underlying bandit problem $\mB$ for good users, and the algorithm, and in particular, they can coordinate. Though the user IDs are fixed, we do not know which users are good and which are adversarial. Moreover, we have no prior information on the underlying shared task of good users in advance. 

The only assumption we require for the arrival process across all users, is that the difference in user frequency is controlled. 
\begin{assumption}
    \label{assumption:user_selection}
    Let $n_i(t)$ be the number of times that the $i^{th}$ user has interacted with the environment up to time step $t$. Then, there exists a universal constant $c_L = O(1)$ such that for all $t \ge L$, $\max_{i \in [L]} n_i(t) / \min_{i \in [L]} n_i(t) \le c_L$.
\end{assumption}
{\bf Remark}. The exact constant $c_L$ factors into all of our results as a linear multiplier, as we simply need to wait for the slowest user to accumulate the required interactions. To simplify the notation, we simply take it to equal 1. 
Our goal is to find an $(\epsilon, \delta)$ provably-approximately-correct (PAC) optimal policy, which we refer as a near-optimal policy, defined as follows:
\begin{definition}
    An algorithm is $(\epsilon, \delta)$-PAC if it returns a policy $\hat{\pi}$ such that
    $\PP(V_\mB^* - V_\mB^{\hat{\pi}} \le \epsilon) \ge 1 - \delta$.
\end{definition}
Note that the optimal policy $\pi^*$ is a stationary policy that satisfies $\pi^*(s) = arg \max_{a \in \mA} \mu(s,a)$. The performance of the algorithm is measured in terms of {\it the number of per-user interactions:} $N = T/L$ where $T$ is the total number of steps the system has taken, to return an $(\epsilon, \delta)$-PAC policy. We will only consider the case $\epsilon < \alpha$, as otherwise the problem is straightforward ({\it e.g.,} we can simply ignore adversaries).

\paragraph{Notation.} Let $Unif(\mA)$ be a uniform distribution over $\mA$.
We use $B(p)$ to denote a Bernoulli distribution with parameter $p \in [0,1]$. For any probability distribution $\mD$, we use $\mD^{\bigotimes n}$ to mean a $n$-product distribution of $\mD$. $d_{TV} (\mD_1, \mD_2)$ is a total-variation distance between two probability distributions $\mD_1, \mD_2$. We interchangeably use $\nu$ and $\mu$ as probability distributions and vectors of probabilities.  

\section{Warm Up: Two Base Cases}
\label{section:baseline}

In this section, we focus on two base cases when either $(i)$ the number of contexts is small, or $(ii)$ the number of actions is small. We develop an $(\epsilon, \delta)$-PAC algorithm robust to adversarial users for arbitrary accuracy $\epsilon > 0$ and small failure probability $\delta > 0$. Together, these two results show that  
$\tilde{O}(\min(S,A) \cdot \alpha / \epsilon^2)$ number of per-user interactions are enough to obtain an $\epsilon$-optimal policy. This result sets the stage for one of the main results of this paper, that establishes a nearly matching lower bound (Section \ref{section:lower_bound}).


\subsection{Multi-Armed Bandit: $S = O(1)$, $A \gg 1$} \label{sec: univar estimator}

We first consider a special case of contextual bandits, with very few ($O(1)$) contexts; thus $\mB$ is essentially a multi-armed bandit problem with many arms. For this case, we can use any efficient {\it univariate robust estimator}, to calculate the mean-reward of each arm. Once we have that, we can play (nearly) optimally. 

The details are as follows. First, assume $S = 1$. Divide the $L$ (good and adversarial) users randomly into $A$ groups, $\mI_a$, $a \in \mA$. Whenever we see a user from $\mI_a$, we play $a \in \mA$. After $T$ total plays, we compute the empirical mean of each user's rewards. Then we take $\hat{\mu}(a)$ to be the $\alpha$-trimmed-mean of the empirical estimates of each user in $\mI_a$ \cite{lugosi2021robust}. 
Precisely, we have the following result. 
\begin{proposition}
    \label{lemma:base_univariate}
    Let the number of users be at least $L = \Omega(A \log A / \alpha)$ and the adversarial rate be $\alpha < 1/3$. Let $\hat{\mu}$ denote the vector of mean estimates, produced by the procedure outlined above. Then after $T / L = O \left(\alpha / \epsilon^2 \right)$ per-user interactions, with probability at least $9/10$, $|\hat{\mu} (a) - \mu (a)| \le \epsilon$ for all $a \in \mA$. 
\end{proposition}
Once we obtain a set of estimated mean-rewards, our policy is immediate: pick the arm $\hat{a} = arg\max_{a \in \mA} \hat{\mu}_a$. This is a $(2\epsilon, 0.1)$-PAC guaranteed policy. 

The procedure can be easily extended to constantly many contexts $S = O(1)$: for each context, we play the same procedure independently, {\it i.e.,} a user arrival $i_t$ with different $s_t$ can be handled in a separate procedure. This would require $T/L = O(\alpha S / \epsilon^2)$ per-user interactions (assuming a uniform distribution over contexts) to obtain an $\epsilon$-optimal policy, therefore achieving the $O(1/A)$ parallelization gain. 

The proof of the above proposition requires us to show that the median-of-means estimator can obtain an $\epsilon$-accurate estimate in the face of as much as an $\alpha$-fraction of corruptions. Thus, it is essentially an immediate corollary of, for instance, the following result from \cite{lugosi2021robust}, which directy applies to the trimmed-mean estimator:
\begin{theorem}[Theorem 1 in \cite{lugosi2021robust}]
    \label{theorem:univariate_robust_mean}
    Let $\mD$ be a distribution on $\mathbb{R}$ with unknown mean $\mu$ and finite variance $\sigma^2$. Let $\alpha < 1/3$. Given an $\alpha$-corrupted set of $L = \Omega(\log(1/\delta) / \alpha)$ samples drawn from $\mD$, then the $\alpha$-trimmed mean produces $\hat{\mu}$ such that with probability at least $1 - \delta$, we have $|\hat{\mu} - \mu| \le O(\sigma \sqrt{\alpha})$.
\end{theorem}

\subsection{Many Contexts, A Few Actions: $S \gg 1$, $A = O(1)$} \label{sec: robust high dimension}
The second baseline is a contextual bandit case with many contexts and constant number of arms. Let us assume here that $\nu(s) = 1/S$, {\it i.e.,} uniform over all contexts. The main challenge here is that we cannot estimate a mean-reward accurately enough from a single user for a fixed state-action pair, since the same context is unlikely to be seen more than once unless we interact with the user long enough (that is, unless $T/L \ge S$). Therefore, a univariate robust estimation approach would not work in this case.


In this scenario, we can use a {\it robust high-dimensional estimator} ({\it e.g.,} \citet{cheng2019high, diakonikolas2019robust}) to estimate mean-rewards. Specifically, let $n_i(T) := \sum_{t=1}^T \indic_{i_t = i}$ be a set of time steps in which the system interacts with the $i^{th}$ user. We collect data by simply playing a randomly selected action at every step. We then estimate the vector $\hat{\mu}_i$ for all users:
\begin{align}
\label{eq:mui_estimate_highdim}
    \hat{\mu}_i =  \frac{SA}{n_i(T)} \sum_{t: i_t = i} r_{t} \cdot \bm{e}_{(s_{t}, a_{t})},
\end{align}
where $\bm{e}_{(s,a)}$ is a standard basis vector with $1$ at position $(s,a)$. For good users $i \in \mI^*$, we can see that the first- and second-order moments of a random quantity $\hat{\mu}_i$ satisfy: $\Exs[r_t \cdot \bm{e}_{(s_{t}, a_{t})}] = \frac{1}{SA} \mu, \ Cov(r_t \cdot \bm{e}_{(s_{t}, a_{t})}) \preceq \frac{1}{SA} I,$ and thus 
\begin{align}
    \Exs[\hat{\mu}_i] = \mu, \ Cov(\hat{\mu}_i) \preceq \frac{L}{T} SA \cdot I. \label{eq:mean_var_check_highdim}
\end{align}
Hence after $T$ steps, we can equivalently consider a set of $\{\hat{\mu}_i\}_{i=1}^L$ as $\alpha$-corrupted $L$ independent samples from the same mean-$\mu$ distributions with bounded second-order moments. Then, we can use an efficient robust mean estimator for a distribution with bounded second moments developed in \cite{cheng2019high}.


The details are as follows. Every time step $t \in [T]$ a user $i_t$ arrives with context $s_t$, and we play a random action $a_t \sim Unif(\mA)$. After $T$ steps, we construct $\hat{\mu}_i$ as in \eqref{eq:mui_estimate_highdim} for all users $i \in [L]$, and run the high-dimensional robust estimator whose existence is guaranteed in Theorem \ref{theorem:high_dim_robust_estimator} below, with input $\{\hat{\mu}_i\}_{i=1}^L$. The quality of the estimated mean-rewards $\hat{\mu}$ is guaranteed by the following:
\begin{proposition}
    \label{lemma:base_high_dim}
    Let the number of users be at least $L = \Omega(SA \log(SA) / \alpha)$ and the adversarial rate be $\alpha < 1/3$. After $T$ steps, with probability at least $9/10$, $\|\hat{\mu} - \mu\|_2 \le O\left(\sqrt{\alpha LSA / T} \right)$.
\end{proposition}
Given this estimate we output a policy $\hat{\pi}$ such that $\hat{\pi}(s) = arg\max_{a \in \mA} \hat{\mu}{(s,a)}$ for all $s \in \mS$. Then a simple algebra shows that (see Appendix \ref{appendix:proof_section_3}),
\begin{align}
    V_{\mB}^* - V_{\mB}^{\hat{\pi}} \le \frac{2}{\sqrt{S}} \|\mu - \hat{\mu}\|_2 \le O \left(\sqrt{\alpha LA / T} \right).\label{eq:high_dim_est_value_error} 
\end{align}
Thus after $T/L = O(\alpha A / \epsilon^2)$ per-user interactions, we obtain an $(\epsilon, 0.1)$-PAC guaranteed policy with the $O(1/S)$ parallelization gain. Note that the failure probability is chosen for the analysis purpose and can be arbitrarily improved to $1-\delta$ guarantee by collecting $\log (1/\delta)$ times more interactions and/or repeating the procedure $\log (1/\delta)$ times. 

The proof of Proposition \ref{lemma:base_high_dim} is essentially a corollary of the following key result on robust mean estimation.

\begin{theorem}[Theorem 1.3 in \cite{cheng2019high}]
    \label{theorem:high_dim_robust_estimator}
    Let $\mD$ be a distribution on $\mathbb{R}^d$ with unknown mean $\mu$ and bounded covariance $\Sigma$ such that $\Sigma \preceq \sigma^2 I$. Given an $\alpha$-corrupted set of $L = \Omega(d \log (d) / \alpha)$ samples drawn from $\mD$ with $\alpha < 1/3$, there is an algorithm that runs in time $\tilde{O}(Ld / \alpha^6)$ and outputs a hypothesis vector $\hat{\mu}$ such that with probability at least $9/10$, it holds $\|\hat{\mu} - \mu\|_2 \le O(\sigma \sqrt{\alpha})$.
\end{theorem}

\section{Lower Bound}
\label{section:lower_bound}
Through the two warm-up cases, we have seen that $\tilde{O}(\min(S,A) \cdot \alpha / \epsilon^2)$ number of per-user interactions are enough to obtain an $\epsilon$-optimal policy. A natural follow-up question is whether we can improve the sample complexity when there are large number of both contexts and actions $S, A \gg 1$. In this section, we show that this is the best sample complexity we can achieve when all context probabilities are uniform, {\it i.e.,} $\nu(s) = 1/S, \forall s \in \mS$:
\begin{theorem}
    \label{theorem:lower_bound}
    Suppose $\alpha < 1/3$ and $L \le poly(S, A, 1/\epsilon)$. For any constant $\beta > 0$, there exists a set of $\alpha$-corrupted multi-user systems such that no algorithm with $T/L = O \left( \min(S,A)^{1-\beta} \cdot \alpha^2 / \epsilon^2 \right)$ per-user interactions can output an $\epsilon$-optimal policy with probability more than $2/3$. 
\end{theorem}
By Theorem \ref{theorem:lower_bound}, we need at least $\Omega \left(\min(S,A) /\epsilon^2 \right)$ per-user interactions (up to inverse logarithmic factors) to obtain an $\epsilon$-optimal policy with a meaningfully large probability. Note that any algorithm that succeeds with a constant probability can be boosted to a high-probability algorithm by repeating the same procedure, and thus our lower bound also implies the non-existence of an algorithm with a constant success probability with $o(\min(S,A))$ per-user interactions. 
\begin{figure}[t]
    \centering
    \includegraphics[width=0.35\textwidth]{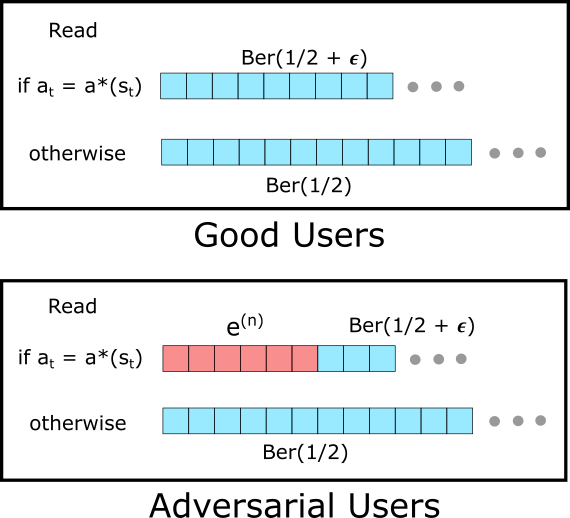}
    \caption{The lower bound construction. The tape describes the data gathered while taking action $a$. \emph{(Blue)}: indicates an uncorrupted reward, and \emph{(Red)}: indicates corrupted reward for the optimal action. For all non-optimal actions, all users generate rewards from $B(1/2)$.}
    \label{fig:lower_bound_fig}
    \vskip -0.3cm
\end{figure}
We prove this result by constructing two systems that are statistically indistinguishable: one is a completely random system (A) where all users only generate rewards from $B(1/2)$. Thus no algorithm can do better than a complete random guess of optimal actions in this system. In the other system (B), good users sample a reward from $B(1/2 + \epsilon)$ for the optimal action given a context, while sampling a reward from $B(1/2)$ for all other actions. We show that in system (B), adversaries can generate sequences of rewards for the optimal actions such that, without the identity of users (whether they are good or adversarial), the collection of all users' reward sequences are statistically indistinguishable in both systems. We then show that this implies the there cannot be an $(\epsilon/3, 1/3)$-PAC algorithm.

\subsection{Lower Bound for Two-Armed Bandit Case}
The proof of Theorem \ref{theorem:lower_bound} starts from a standard construction for the two-armed bandit problem that highlights the necessity of $\Omega(1/\epsilon^2)$ per-user interactions (up to inverse $\log(1/\epsilon)$ factors) for $S = 1, A = 2$. Let $\frac{1}{L} \le \alpha < \frac{1}{6}$, $L \le \epsilon^{-O(1)}$, and $\epsilon < \frac{0.01 \alpha}{\log L}$ be sufficiently small.

Consider a base system (A) where the reward of all users is sampled from $B(1/2)$ regardless of arms played. Further suppose that a virtual optimal arm $a^* \in \mA$ is selected independently and uniformly at random. Clearly, we cannot guess $a^*$ with probability better than $1/2$ since the observed rewards are statistically independent from $a^*$. 

Then, consider the following procedure which generates a corrupted multi-user system (B) as follows (Figure \ref{fig:lower_bound_fig}): 
\begin{enumerate}
    \item For every $i \in [L]$, with probability $1-\alpha$, add $i$ to $\mI^*$, {\it i.e.,} $i$ is a good user. Otherwise, $i$ is an adversarial user.
    \item If $i \in \mI^*$, then a user $i$ returns a reward $r \sim B(1/2 + \epsilon)$ whenever $a^*$ is played, and $r \sim B(1/2)$ otherwise.
    \item If $i \notin \mI^*$, then if the chosen action is $a^*$, it reads the reward sequence $e^{(n)}$ for the first $n$ times, where $e^{(n)}$ is sampled from the distribution $E^{(n)}$ (described in the proof Lemma~\ref{lemma:contamination_tv}). If the action is different than $a^*$, then it samples a reward $r \sim B(1/2)$.
\end{enumerate}
Note that the corruption by an adversarial user happens only for the first $n$ times and only for the optimal action. Also note that since the true identity of users are decided following $B(\alpha)$ independently, with high probability, system (B) generates a $2 \alpha$-corrupted multi-user system. 

Now we show that adversaries can play reward sequences such that systems (A) and (B) are not distinguishable. We first need the following lemma on confusing the sequence of product distributions:
\begin{lemma}
    \label{lemma:contamination_tv}
    Let $n \le 0.01 \alpha^2 / (\epsilon^2 \log L)$. Then there exists a distribution $E^{(n)}$ over $\{0,1\}^{n}$ such that
    \begin{align}
        d_{TV} \Big( B(1/2)^{\bigotimes n}, \ (1 - \alpha) \cdot B(1/2 + \epsilon &)^{\bigotimes n} + \alpha \cdot E^{(n)} \Big) \nonumber \\
        &\le 1/L^4. \label{eq:tv_craft}
    \end{align}
\end{lemma}
Note that $B(1/2)^{\bigotimes n}$ corresponds to a distribution of length $n$ reward sequence of any user in the base system. 
The key fact here is that, 
if we can only see up to length-$n$ sequences of rewards from all $L$ users, then due to Lemma \ref{lemma:contamination_tv} and Le Cam's two-points method \cite{lecam1973convergence}, we cannot identify whether a system is (A) or (B) with probability better than $1/2 + 1/L^3 \approx 1/2$. 

Now we show by contradiction that no algorithm that interacts with each user less than $n$ times can recover $a^*$ in system (B) with probability more than $2/3$. To this end, suppose there exists an $(\epsilon/3, 1/3)$-PAC algorithm with at most $n$ per-user interactions. Then using this algorithm, we can conduct a statistical test that recovers the identity of the system (whether it is system (A) or (B)). But this contradicts the information theoretical limit. The test goes as follows. 
Let $\widehat{a}\in \mA$ be the optimal action (the one with the most probability assigned in the returned policy) that the algorithm outputs. If $\widehat{a} = a^*$, then output (B), otherwise output (A). The success probability of this testing is at least $1/2 \cdot 1/2 + 1/2 \cdot 2/3 \approx 7/12$. This contradicts Le Cam's lower bound. In other words, any $(\epsilon/3, 1/3)$-PAC algorithm must require at least $\tilde{\Omega} \left( \alpha^2 / \epsilon^2 \right)$ per-user interactions.

\subsection{$\min(S,A)$ Lower Bound}

We turn our attention to the general case with large $S$ and $A$. We still consider two systems (A) and (B), with the only difference that now an optimal action $a^*(s) \in \mA$ depends on the context $s \in \mS$. For all $s \in \mS$, let each $a^*(s)$ be decided uniformly at random, independently of any other events. The key idea for getting $\min(S,A)$ lower bound is to check that it is impossible to play better than a random guess on $a^*(\cdot)$ in system (A). Furthermore, if we cannot play $a^*(\cdot)$ more than $n$ times with any user in system (A), then by the similar contradiction argument to Le Cam's fundamental limit, we cannot play the optimal action $a^*(\cdot)$ more than $n$ times for any user either in system (B). Here, following Lemma \ref{lemma:contamination_tv}, we let $n = o( \alpha^2 / \epsilon^2)$.  

Note that $a^*(s)$ can be any action with probability $1/A$ for every context $s$, therefore a random guess on $a^*(\cdot)$ is correct with probability $1/A$. In the setting where $S \gg A$, within each user we see any context essentially only once, and only with $1/A$ probability the optimal action can be chosen. Hence we need at least $N = O(An)$ interactions to ensure at least $n$ correct actions chosen for any user. On the other hand, if $A \gg S$, then from a random guess, one can show that it is unlikely to guess the optimal actions for more than constant number of contexts within $L = A^{O(1)}$ trials. In other words, for each user we would play the optimal actions for only constant number of contexts, and such contexts are only observed $O(N/S)$ times after $N$ per-user interactions. Thus, we need at least $N = O(Sn)$ interactions to play optimal actions at least $n$ times.

Therefore, we can conclude that we need at least $\Omega(\min(S,A) \cdot \alpha^2 / \epsilon^2)$ per-user interactions in order to play $a^*(\cdot)$ more than $n$ times. Furthermore, we show that this is also the case in system (B) due to Le Cam's two point method, which can be translated to impossibility of $(\epsilon/3, 1/3)$-PAC algorithm. Complete and formal proof of Theorem \ref{theorem:lower_bound} is provided in Appendix \ref{appendix:proof_lower_bound}.

We comment that our lower bound holds only when $L$ is polynomial in $S$ and $A$. This is required because, if we allow arbitrarily large number of users, {\it e.g.,} $L \gg A^S$, then we can divide users into $A^S$ groups and in each group we can evaluate every possible stationary policy with $O(1 / \epsilon^2)$ per-user interactions. Theorem \ref{theorem:lower_bound} suggests that such desirable sample-complexity is only possible when more than polynomial number of users are available.

\section{Main Algorithm}
\label{section:main_algorithm}

\begin{algorithm}[t]
    \caption{Estimate $\nu$}
    \label{algo:estimate_nu}
    \begin{algorithmic}[1]
        \FOR{$t = 1, 2, ..., T_0$}
            \STATE{An user $i_t$ arrives with a context $s_t$.}
            \STATE{Play any action and move to next time step.}
        \ENDFOR
        \STATE{Let $\sigma_s^2 := \max \left( n_s(T_0), O(\log S) \right) / T_0$ for all $s \in \mS$.}
        \STATE{Compute $b_i = \frac{1}{n_i(T_0)} \sum_{t: i_t = i} \sigma_{s_t}^{-1} \bm{e}_{s_t}$ for all $i \in [L]$.}
        \STATE{Get $\hat{b}$ with input $\{\hat{b_i}\}_{i=1}^L$ to high-dimensional robust estimator described in Theorem \ref{theorem:high_dim_robust_estimator}.}
        \STATE{Return $\hat{\nu} := diag(\{\sigma_s\}_{s \in \mS}) \cdot \hat{b}$.} 
    \end{algorithmic}
\end{algorithm}

In Section \ref{section:baseline}, we showed that the best strategy is simply using either univariate robust estimator if $A \ge S$, or high-dimensional robust estimator $S \ge A$ when the context probability is uniform, {\it i.e.,} $\nu(s) = 1/S$ for all $s \in \mS$. We extend this approach to the case where $\nu(s)$ is a general distribution over contexts. Non-uniform context probability often arises in recommendation systems where a certain set of contexts (e.g., keywords) are more preferred by users.

Our main idea to exploit non-uniform context probability is quite simple: suppose we can order contexts in probability descending order $s_{[1]},\ldots, s_{[S]}$ such that $\nu(s_{[1]}) \ge \ldots \ge \nu(s_{[S]})$. Let $\mS_+ := \{s_{[j]} \mid j \in [\min(A,S)] \}$ be a set of top-$A$ frequent contexts. Note that if $A \ge S$, then $\mS_+ = \mS$. Now for the estimation of rewards under contexts $s \in \mS_+$, we use a uni-variate robust estimator (Section~\ref{sec: univar estimator}). For the other contexts $s \in \mS_- := \mS / \mS_+$, we estimate $\mu$ for the remaining part using a high-dimensional robust estimator (Section~\ref{sec: robust high dimension}). By combining the two robust estimators in a corrupted multi-user system, we can achieve the improved sample complexity in polynomial time. 

One challenge is that we are not given the probability distribution over contexts $\nu$ observed by good users. Thus, in the pre-processing step, we need to estimate $\nu$ to decide which contexts we consider as top-$A$ contexts, {\it i.e.,} find $\hat{\nu}$ such that $\|\hat{\nu} - \nu\|_1 \le \epsilon$. This can be done by the robust high-dimensional estimator with re-scaling of each coordinate after $T_0 = O(\alpha L/\epsilon^2)$ time steps. Let $n_s(T_0) := \sum_{t=1}^{T_0} \indic_{s_t = s}$ be the total number of times that a context $s$ is observed, and $\sigma_s^2 := \max( n_s(T_0), O(\log S) ) / T_0$. Note that $\sigma_s^{-1}$ scaling serves as a equalizer of context-wise variances. We summarize the procedure to estimate $\nu$ in Algorithm \ref{algo:estimate_nu}. 

Note that the samples $\{\hat{b}_i\}_{i=1}^L$ are not independent to each other. Nevertheless, robust estimators in \citet{cheng2019high, lugosi2021robust} can still be used when samples are not exactly identical or independent as long as some deterministic conditions hold (see Appendix \ref{appendix:condition_for_estimators}). We show that with this nice property, we can still recover good enough estimates of $\hat{\nu}$ in Algorithm \ref{algo:estimate_nu}. 

We conclude this section with a theoretical guarantee on ($\epsilon$, 1/3)-optimality of the policy returned by the main Algorithm \ref{algo:nonuniform_robust} (Robust MCB) combining all components:
\begin{theorem}
    \label{theorem:upper_bound}
    Let $\alpha < 1/3$ and $L = \Omega(SA \log(SA) / \alpha)$. If we run Algorithm \ref{algo:nonuniform_robust} with $T_0 = O(L \alpha / \epsilon^2)$ and $T > L \cdot \alpha /\epsilon^2$, then with probability at least 2/3, $\hat{\pi}$ satisfies:
    \begin{align}
        V^*_{\mB} - V_{\mB}^{\hat{\pi}} \le O \left ( K(\mB) \cdot \sqrt{\alpha L / T} \right),
    \end{align}
    where $K(\mB)$ is an instance-dependent quantity given by
    \begin{align}
        K(\mB) = \sum_{j=1}^{A} \sqrt{\nu(s_{[j]})} + \sqrt{\sum_{j=A+1}^S \nu(s_{[j]}) \cdot A} .
    \end{align}
\end{theorem}
For the uniform context probability, Theorem \ref{theorem:upper_bound} provides a PAC guarantee with $O(\min(S,A) \cdot \alpha / \epsilon^2)$ per-user interactions. The benefit of non-uniform context probabilities is more explicit in the following example: suppose $\nu(s_{[j]}) \propto 1/j^{1+\gamma}$ for $\gamma > 0$, {\it i.e.,} context probability decays in polynomial rates. Then the sample complexity per-user guaranteed by Algorithm \ref{algo:nonuniform_robust} is $O(\min(S,A)^{1-\gamma} \cdot \alpha / \epsilon^2)$ for $\gamma < 1$. When $\gamma \ge 1$, we have a desired per-user sample-complexity $\tilde{O} (\alpha / \epsilon^2)$.

\begin{algorithm}[t]
    \caption{Robust Multi-task Contextual Bandits}
    \label{algo:nonuniform_robust}
    \begin{algorithmic}[1]
        \STATE{Pre-Process 1: Get top-$A$ frequent contexts $\mS_+$ from $\hat{\nu}$ returned by Algorithm \ref{algo:estimate_nu}.}
        \STATE{Pre-Process 2: For each $s \in \mS_+$ and $i \in [L]$, assign $a_{i,s} \sim Unif(\mA)$. Add $i$ to $\mI_{s, a_{i,s}}$.}
        \FOR{$t = 1, 2, ..., T$}
            \STATE{An user $i_t$ arrives with a context $s_t$.}
            \IF{$s_t \in \mS_+$}
                \STATE{Play $a_t = a_{i_t, s_t}$ and observe reward $r_t$.}
            \ELSE
                \STATE{Play $a_t \sim Unif(\mA)$ and observe reward $r_t$.}
            \ENDIF
        \ENDFOR
        \STATE{Retrieve $\hat{\mu}$ from the procedure in Appendix \ref{appendix:complete_procedure}.}
        \STATE{Return $\hat{\pi}$ s.t. $\hat{\pi}(s) = arg\max_{a \in \mA} \hat{\mu} (s,a)$ $\forall s \in \mS$.}
    \end{algorithmic}
\end{algorithm}

\begin{figure*}[t!]
    \centering
    \begin{tabular}{cccc}
        \includegraphics[width=38mm]{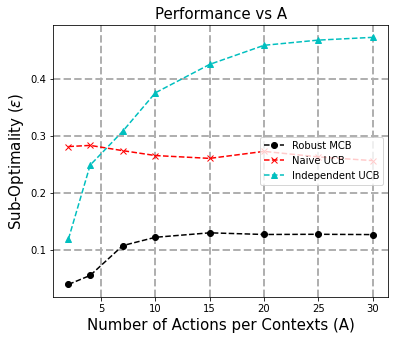} & 
        \includegraphics[width=38mm]{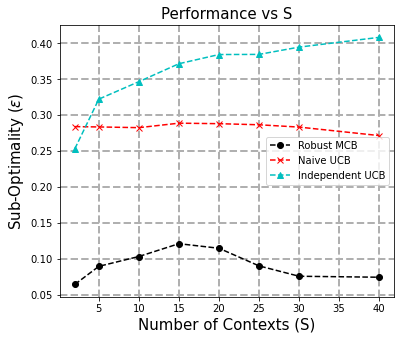} & 
        \includegraphics[width=38mm]{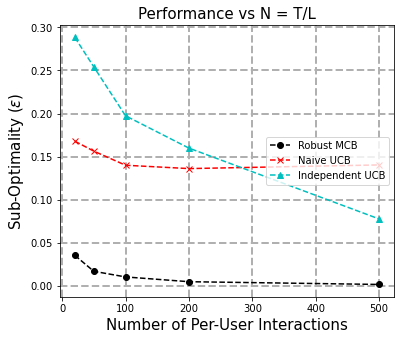} & 
        \includegraphics[width=38mm]{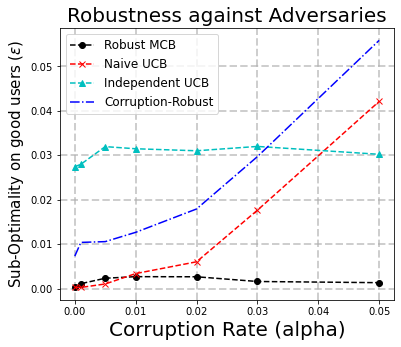} \\
        (a) & (b) & (c) & (d)
    \end{tabular}
    \caption{(a-c) Sub-optimality of the returned policy $\epsilon$ depending on each parameters $S, A, T/L$. Y-axis illustrates the sub-optimality of the policy, {\it i.e.,} $\epsilon \propto \sqrt{\min(S,A) \cdot T / L}$.
    (d) Robustness of algorithms in $\alpha$-corrupted multi-user systems. 
    }
    \label{fig:sample_complexity}
    \vskip -0.2cm
\end{figure*}

\subsection{Discussion}
We remark here a few discussion points that could be of independent and future interest.

\paragraph{Regret Minimization:} In the multi-armed bandit case (or $S \le A$), we can implement a simple regret-minimization algorithm by combining doubling trick and successive elimination techniques. The challenge arises in the other case $S \ge A$. Our main idea for handling this case is to use a high-dimensional robust estimator. However, the guarantee for mean-reward estimators are only given in overall $l_2$-distances, which guarantee performance of returned policies {\it only in expectation}. However, low-regret algorithm should be able to eliminate sub-optimal actions, which needs coordinate-wise accurate estimates of mean-rewards which is not available for high-dimensional robust estimation. 

One way to minimize the {\it overall} regret is to divide user sets into two different groups: one group for exploration, and the other one for exploitation. If the number of users are given $\Omega(S^{3/2} A)$, then we can interact with $O(1/\sqrt{S})$-fraction of users to find an improved policy, while for others we play a policy obtained in the previous round of epoch. It would be an interesting question to find an algorithm that performs uniformly good for all users in terms of regret. 

\paragraph{Tightness on $\alpha$:} Theorem \ref{theorem:lower_bound} suggests $\Omega(\alpha^2)$ lower bound while our upper bound is guaranteed with $O(\alpha)$ samples. This $\alpha$-gap in the lower and upper bounds results from the fact that our lower bound is built upon Bernoulli reward assumptions, and thus holds for sub-Gaussian type reward distributions, while our upper bound relies only on the bounded second-order moment condition for reward distributions. Tightening a factor of $\alpha$ for sub-Gaussian reward distributions might require developing robust estimators for sub-Gaussian distributions with unknown bounded covariances.

\paragraph{Similar Preferences:} For personalized recommendations, good users may have similar but not exactly the same preference over all contexts and items \cite{ghosh2021collaborative}. For instance, suppose that for all $i \neq j$ and $i, j \in \mI^*$ with underlying tasks $\mB_i, \mB_j$ respectively, and we have $\forall \pi \in \Pi, |V_{\mB_i}^{\pi} - V_{\mB_j}^{\pi}| \le \epsilon_0$. We mention here that the robust estimators we employ here are robust to small perturbations in samples. Therefore, we can first run Algorithm \ref{algo:nonuniform_robust} to find a common policy $\bar{\pi}$ such that $|V_{\mB_i}^* - V_{\mB_i}^{\bar{\pi}}| \le O(\epsilon_0 + \epsilon)$, after which we can learn for each user separately, {\it e.g.,} using the algorithm in \cite{ghosh2021problem}.

\section{Experiments} 
\label{section:experiments}

We evaluate the proposed algorithm on synthetic data. We set the sub-optimality gap in all contexts approximately~$0.3$. Our first experiment illustrates the performance of Robust MCB (Algorithm \ref{algo:nonuniform_robust}) in terms of the sub-optimality $\epsilon$ of a returned policy for various numbers of users, contexts, actions, per-user interactions and corruption-rates. Additional experiments on the rate of adversaries and similar preferences are presented in Appendix \ref{appendix:additional_experiments}.

\paragraph{Sample Complexity.} 
We first check the sample complexity dependence on $S$ and $A$ as stated in Theorem \ref{theorem:upper_bound}. We compare our robust multitask contextual bandit (Robust MCB) algorithm to two primitive algorithms: (i) an UCB algorithm that does not share information across users (Independent UCB) and (ii) the UCB algorithm that ignores user identifiers as if there is no adversary (Naive UCB). The performance of policy is evaluated after a certain number of time steps on a good user. We generate random instances of multitask contextual bandits on various number of contexts, actions, and the number of users. The fraction of adversaries is 20 percent, {\it i.e.,} $\alpha = 0.2$. The measured sub-optimal gaps $\epsilon$ are averaged over 50 independent experiments. 

The experimental results are given in Figure \ref{fig:sample_complexity} (a)-(c). We fix the base parameters as $S = 10$, $A = 10$, and $T/L = 30$ with $L = O(SA \log(SA))$ users, and measure the accuracy of returned policies for varying parameters. (c) shows how the number of per-user interactions contributes to the performance of returned policies. As shown in the figure, our robust method outperforms two naive approaches. More importantly, we can observe that the increase in $S$ or $A$ does not degrade the performance which confirms $\min(S,A)$ dependency on the sample complexity.

\paragraph{Corruption Robust Algorithms.} Although not considered in the framework of multitask learning, it is worth considering corruption-robust algorithms to defend adversaries' plays \cite{lykouris2018stochastic, liu2021cooperative}. We implement the robust algorithm in \cite{liu2021cooperative} without communication constraint, and compare the performance as increasing the corruption rate $\alpha$ fixing $S = 2$, $A = 5$, $L = 500$ and $T/L = 500$. Since the amount of corruption is linear in $T$, a small fraction of adversarial users can easily attack the corruption-robust algorithm, while we defend against such attacks by exploiting the side information $i_t$, the unique identifiers of data source (Figure \ref{fig:sample_complexity} (d)).

\section{Future Work}
We believe that our work opens up the prospect of investigating more general problems of multitask learning with adversarial users. In particular, we believe coordinated attacks on Markov decision processes (MDPs) would be an interesting future problem to explore. Extending our results in the context of function approximation would also be an interesting future direction. In more technical directions, it would be also interesting to study tighter instance-dependent sample-complexity and regret minimization algorithms.

On a deeper level, the type of questions we have asked  should be put in the context of responsible AI. 
Our stylized model provided principled answers to questions such as how many individuals must collude to manipulate a decision? And how to effectively address the possibility of collusion between agents? We show that by hardening a decision algorithm, it is possible to overcome collusion of a much larger portion of the population.

\bibliographystyle{icml2022}
\bibliography{main}

\newpage

\begin{appendices}
\onecolumn
\section{Additional Experiments}
\label{appendix:additional_experiments}
\begin{figure*}[!t]
    \centering
    \begin{tabular}{ccc}
        \includegraphics[width=79mm]{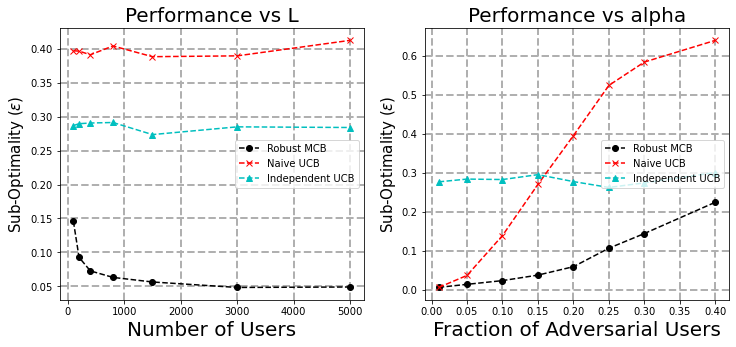} &
        \includegraphics[width=41mm]{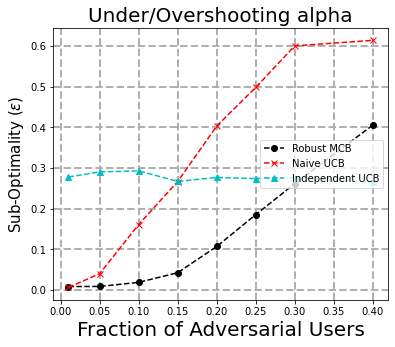} &
        \includegraphics[width=41mm]{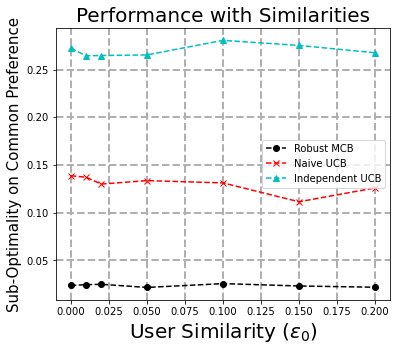} \\
        (a) & (b) & (c)
    \end{tabular}
    \caption{(a) Increase in the number of users only helps until $L = \tilde{O}(SA/\alpha)$. Afterwards, the corruption rate rules the required number of per-user interactions. (b) Robust MCB is tolerant to small differences in user preferences.}
    \label{fig:advanced_experiment}
\end{figure*}

\paragraph{Effective Corruption-Rate.} We recall that we require the number of users to be $\tilde{\Omega}(SA / \alpha)$ for robust estimation rewards. When $L$ is small compared to inverse of $\alpha$, we may consider $\alpha' = \max \left( \frac{SA}{L}, \alpha \right)$ to be an effective fraction of adversarial users. This effect can be seen in Figure \ref{fig:advanced_experiment} (a): we fix $S = 20, A = 10, T/L = 50$, and see how the policy is improved as $L$ increases. The effective corruption rate decreases when $L$ is below a threshold, and we see the improvement in the returned policy. When $L$ becomes larger the threshold, increase in the number of users no longer improves the quality of the policy. In such case, we can only obtain a better policy by collecting more data from individual user.

\paragraph{Under/Over-Shooting Corruption Rate.} We also perform a ablation study when we misspecify the hyper-parameter for the corruption rate $\alpha$. We set our algorithm to be run with $\alpha = 0.1$, when the real fraction of adversaries varying from 0 to 0.3. When the actual corruption rate is less than the hyper-parameter $\alpha = 0.1$, the algorithm is robust and still outperforms other methods. When the actual corruption rate is larger than $0.1$, our algorithm starts to lose the robustness against adversaries. We can conclude that over-shooting is always safer than under-shooting the rate of corruptions.

\paragraph{Similar Preferences.} We make users' preferences slightly different from each other by adding small random perturbations of magnitude $O(\epsilon_0)$ to mean-rewards of all contexts and actions for each user. As we gradually increase the amount of perturbations, Robust MCB can tolerate small differences and find the best policy for the common preferences (Figure \ref{fig:advanced_experiment} (c)).

\section{Deferred Proofs for Section \ref{section:baseline}}

\subsection{Conditions for Robust Estimators}
\label{appendix:condition_for_estimators}
As described in recent work for high-dimensional robust estimations \cite{diakonikolas2019robust, cheng2019high, lugosi2021robust}, Theorem \ref{theorem:univariate_robust_mean} and \ref{theorem:high_dim_robust_estimator} only require a few deterministic conditions to succeed on the concentration of the first and second-order moments. For these robust estimators for $d$-dimensional samples, we need the following deterministic conditions:
\begin{condition}
    \label{condition:deterministic_robust}
    Suppose $N$ good sample points $X_1, X_2, ..., X_N \in \mathbb{R}^d$ (not necessarily identically or independently distributed). Let $\mathcal{W}_{\epsilon} = \left\{ (w_1, w_2, ..., w_N) | \sum_{i=1}^N w_i = 1, \ 0 \le w_i \le \frac{1}{(1-\epsilon)N } \right\}$ for a fixed $\epsilon < 1/3$. Then for all $w \in \mathcal{W}_{3\epsilon}$, there exists $\mu^*$ and some absolute constants $\delta_1 = O(\sqrt{\epsilon}), \delta_2 = O(1), \delta_3 = O(\sqrt{d/\epsilon})$ such that
    \begin{align}
        \left\| \sum_{i=1}^N w_i (X_i - \mu^*) \right\|_2 \le \delta_1, \\
        \left\| \sum_{i=1}^N w_i (X_i - \mu^*) (X_i - \mu^*)^\top \right\|_2 \le \delta_2, \\
        \forall i \in [N], \quad \|X_i - \mu^*\|_2 \le \delta_3.
    \end{align}
\end{condition}
Note that i.i.d. assumption on samples is only a sufficient condition for Condition \ref{condition:deterministic_robust} to hold. When $d = 1$, the above condition also subsumes the precondition for a uni-variate robust estimator in \cite{lugosi2021robust}. To ensure the correctness of robust estimators we use henceforth, we only need to check these deterministic conditions. When each $X_i$ is an independent random variable with mean $\mu^*$ and covaraince $Cov(X_i) \preceq I$, following the argument in Appendix A in \cite{diakonikolas2019robust} and Lemma A.18 in \cite{diakonikolas2017being}, we can consider $\{X_i\}_{i=1}^N$ as $(0.1 \epsilon)$-corrupted samples with $\{X_i | i \in [N], \|X_i - \mu^*\|_2 \le \delta_3 \}$ as good sample points that satisfy Condition \ref{condition:deterministic_robust}, with probability at least $1 - \delta$ as long as $N = \Omega(d \log (d / \delta) / \epsilon)$ (by Bernstein's inequality for the mean, and the matrix Chernoff bound \cite{tropp2015introduction} for covariance). Henceforth, we only need to check whether the distribution of each sample point $X_i$ has the common mean and bounded second-order moments. 

\subsection{Proof of Propositions \ref{lemma:base_univariate} and \ref{lemma:base_high_dim}}
\label{appendix:proof_section_3}
For each $a \in \mA$, we only need to check Condition \ref{condition:deterministic_robust} for mean-reward estimates $\hat{r}_i$ from good users $i \in \mI^* \cap \mI_a$. Note that by construction, $|\mI_a| = \Theta(\log(A) / \alpha)$ and $|\mI^* \cap \mI_a| \ge (1 - 1.1 \alpha) |\mI_a|$ with probability more than 0.99 for a sufficiently large $L = \Omega(A \log (A) / \alpha)$. Conditioned on the event that each $n_i \ge O(T/L)$, which holds with probability $1$ under Assumption \ref{assumption:user_selection}, we have
\begin{align*}
    \Exs[\hat{r}_i] = \mu(a), \quad Var(\hat{r}_i) \le O(L/T). 
\end{align*}
We can use either median-of-means for $\alpha \ge 1/16$, or Theorem \ref{theorem:univariate_robust_mean} in \cite{lugosi2021robust} for $\alpha < 1/16$. Then robust estimator outputs $\hat{\mu}(a)$ such that $|\hat{\mu}(a) - \mu(a)| \le O \left( \sqrt{\alpha L / T} \right)$ with probability at least $1 - 0.1 / A$. Taking union bound over all $a \in \mA$, we get the lemma.

The proof for Lemma \ref{lemma:base_high_dim} follows similarly from equation \eqref{eq:mean_var_check_highdim} and Theorem \ref{theorem:high_dim_robust_estimator}.

\paragraph{Proof of equation \eqref{eq:high_dim_est_value_error}: } 
A simple algebra shows that
\begin{align*}
    V_{\mB}^* - V_{\mB}^{\hat{\pi}} &= \frac{1}{S} \sum_{s \in \mS} \mu(s, \pi^*(s)) - \mu(s, \hat{\pi} (s)) \\
    &\le \frac{1}{S} \sum_{s \in \mS} \mu(s, \pi^*(s)) - \hat{\mu}(s, \pi^*(s)) + \hat{\mu}(s, \hat{\pi} (s)) - \mu(s, \hat{\pi} (s)) \\
    &\le \frac{1}{\sqrt{S}} \sqrt{\sum_{s \in \mS} \left(\mu(s, \pi^*(s)) - \hat{\mu}(s, \pi^*(s)) \right)^2} + \frac{1}{\sqrt{S}} \sqrt{\sum_{s \in \mS} \left(\mu(s, \hat{\pi}(s)) - \hat{\mu}(s, \hat{\pi}(s)) \right)^2} \\
    &\le \frac{2}{\sqrt{S}} \|\mu - \hat{\mu}\|_2 \le O \left(\sqrt{\alpha LA / T} \right). 
\end{align*}

\section{Deferred Proofs in Section \ref{section:lower_bound}}
In this appendix, we provide a full proof of our key negative results.
\subsection{Proof of Lemma \ref{lemma:contamination_tv}}
We first note that a simple algebra shows
\begin{align*}
    d_{TV} \left( B(1/2)^{\bigotimes n}, B(1/2 + \epsilon)^{\bigotimes n} \right) \le \sqrt{\frac{n}{2} d_{KL}(B(1/2) || B(1/2 + \epsilon)) } \le 0.1 \alpha,
\end{align*}
where $d_{KL}$ is a Kullback-Leibuler divergence, and we used Pinsker's inequality with $n \le 0.01 \alpha^2 / \epsilon^2$. For any sequence $e^{(n)} \in \{0,1\}^n$, let $\PP_{B(1/2)^{\bigotimes n}} (e^{(n)}) = 2^{-n}$ when sampled from $B(1/2)^{\bigotimes n}$. Let $e^{(n)}_j$ be a value at the $j^{th}$ position in sequence $e^{(n)}$, and let $\Delta(e^{(n)})$ be defined as $\sum_{j=1}^n \indic_{e_j^{(n)} = 1} - \indic_{e_j^{(n)} = 0}$, {\it i.e.,} the differences in the number of $1$'s and $0$'s in $e^{(n)}$. For any $e^{(n)} \in E_{good}^{(n)} := \{ e^{(n)} | \Delta(e^{(n)}) \le 4\sqrt{n \log L} \}$, we can show that
\begin{align*}
    \PP_{B(1/2)^{\bigotimes n}} (e^{(n)}) \ge (1-\alpha) \PP_{B(1/2 + \epsilon)^{\bigotimes n}} (e^{(n)}).
\end{align*}
Let $t = \Delta(e^{(n)}) < 4\sqrt{n \log L}$. The inequality is obvious for $t \le 0$, and for $t > 0$,
\begin{align*}
    \frac{\PP_{B(1/2 + \epsilon)^{\bigotimes n}} (e^{(n)})}{\PP_{B(1/2)^{\bigotimes n}} (e^{(n)})} &= (1-4\epsilon^2)^{(n - t) / 2} \cdot (1 + 2\epsilon)^{t} \\
    &\le (1 - n \cdot \epsilon^2) \cdot (1 + 4 t \epsilon) \le (1 - c \alpha^2) \cdot (1 + 4\sqrt{c} \alpha) \le \frac{1}{1-\alpha}.
\end{align*}
Let the distribution $E^{(n)}$ over $\{0,1\}^n$ such that for $e^{(n)} \in E_{good}^{(n)}$, 
\begin{align*}
    \PP_{E^{(n)}} (e^{(n)}) = \frac{1}{Z} \frac{1}{\alpha} \left( \PP_{B(1/2)^{\bigotimes n}} (e^{(n)}) - (1-\alpha) \PP_{B(1/2 + \epsilon)^{\bigotimes n}} (e^{(n)}) \right),
\end{align*}
where $Z$ is a normalizer to make $E^{(n)}$ a valid distribution, and $\PP_{E^{(n)}} (e^{(n)}) = 0$ for $e^{(n)} \notin E_{good}^{(n)}$. Now to check the equation \eqref{eq:tv_craft}, we see that $Z$ is less than:
\begin{align*}
    Z &= \frac{1}{\alpha} \left(\PP_{B(1/2)^{\bigotimes n}} \left( e^{(n)} \in E_{good}^{(n)} \right) - (1-\alpha) \PP_{B(1/2 + \epsilon)^{\bigotimes n}} \left( e^{(n)} \in E_{good}^{(n)}\right) \right) \\
    &= \frac{1}{\alpha} \left( \PP_{B(1/2)^{\bigotimes n}} \left(e^{(n)} \notin E_{good}^{(n)} \right) - \PP_{B(1/2 + \epsilon)^{\bigotimes n}} \left(e^{(n)} \notin E_{good}^{(n)} \right) + \alpha \cdot \PP_{B(1/2 + \epsilon)^{\bigotimes n}} \left(e^{(n)} \in E_{good}^{(n)} \right) \right) \\
    &\le \frac{1}{\alpha} \PP_{B(1/2)^{\bigotimes n}} \left(e^{(n)} \notin E_{good}^{(n)} \right) + \PP_{B(1/2)^{\bigotimes n}} \left(e^{(n)} \in E_{good}^{(n)} \right) \\
    &\le 1 + \frac{1-\alpha}{\alpha} \cdot \PP_{B(1/2)^{\bigotimes n}} \left(e^{(n)} \notin E_{good}^{(n)} \right). 
\end{align*}
The first inequality comes from the fact that $\PP_{B(1/2 + \epsilon)^{\bigotimes n}} \left(e^{(n)} \in E_{good}^{(n)} \right) \le \PP_{B(1/2)^{\bigotimes n}} \left(e^{(n)} \in E_{good}^{(n)} \right)$. Standard Hoffeding's inequality shows that $\PP_{B(1/2)^{\bigotimes n}} \left(e^{(n)} \notin E_{good}^{(n)} \right) \le \exp(-(4\sqrt{n\log L})^2/ (2n)) = 1/L^8$. For the lower bound on $Z$, we can simply check that
\begin{align*}
    Z &= \frac{1}{\alpha} \left(\PP_{B(1/2)^{\bigotimes n}} \left( e^{(n)} \in E_{good}^{(n)} \right) - (1-\alpha) \PP_{B(1/2 + \epsilon)^{\bigotimes n}} \left( e^{(n)} \in E_{good}^{(n)}\right) \right) \\
    &\ge \frac{1}{\alpha} \left(\PP_{B(1/2)^{\bigotimes n}} \left( e^{(n)} \in E_{good}^{(n)} \right) - (1-\alpha) \PP_{B(1/2)^{\bigotimes n}} \left( e^{(n)} \in E_{good}^{(n)}\right) \right) \\
    &= 1 - \PP_{B(1/2)^{\bigotimes n}} \left(e^{(n)} \notin E_{good}^{(n)} \right). 
\end{align*}
Similarly, we can also show that $Z \ge 1 - \frac{1}{L^8}$. A simple algebra on total-variance distance shows us that $\eqref{eq:tv_craft} \le \frac{1}{L^4}$.

\subsection{Full Proof of Theorem \ref{theorem:lower_bound}}
\label{appendix:proof_lower_bound}
For the construction of hard instances, we assume that $\alpha < 1/6$, $\epsilon \le c \cdot \alpha / \log^{3/2} L$ for a sufficiently small constant $c \le 0.01$ and $L \le \min(S,A)^{O(1)}$. We also assume that a context $s$ is always sampled independently from $Unif(\mS)$ regardless of incoming users and actions played. 

We first show that in system A, it is not possible to play right actions more than $n \le 0.01 \alpha^2/ (\epsilon^2 \log L)$ times for any user $i$ if we interact with user $i$ less than $(\min(S,A) \cdot \alpha^2 / \epsilon^2)$ times. Let $\Eps$ be the event that there exists at least one user $i \in [L]$ such that we played right actions more than $n$ times with $i$ after $N$ per-user interactions, where $N$ is specified in Lemma \ref{lemma:auxiliary_game_A}. In system A, we can show that the chance of event $\Eps$ is very small for any algorithm:
\begin{lemma}
    \label{lemma:auxiliary_game_A}
    In system A, let $\beta > 0$ be any small constant. Let the number of contexts $S$ be sufficiently large so that $\frac{\log S}{\log (\log L)} > 2 / \beta$. Suppose that we interact with any user no more than $N = n \cdot \min(S, A)^{1-\beta}$ times. Then, no algorithm can trigger the event $\Eps$ with probability more than $1/L^2$. 
\end{lemma}
\begin{proof}
    Note that in system A, observed reward sequences are independent of the choice of actions at every step. Consequently, any choice of actions by the agent is independent of correct actions $a^*(s)$. Suppose that an algorithm interacted with all users $N$ times, and let $\{s_{i,t}, a_{i,t}\}_{t=1}^N$ be a length $N$ sequence of contexts and actions when interacting with user $i$. Whatever action choices $\{a_{i,t}\}_{t=1}^N$ made by the agent is statistically independent of $a^*(s)$. Hence we can equivalently think that $a^*(s)$ is a random guess of chosen actions for a context $s$. 
    
    Now we can change the game to guessing more than $n$ chosen actions $a_{i,t}$ by an user $i$ using the completely random guess $a^*(s)$ for all $s \in \mS$. Since $a^*(s)$ is completely independent of interaction histories, without loss of generality, we can assume that $s_{i,t}, a_{i,t}$ are fixed after algorithm interacts $N$ times with every user. If there is at least one user $i$ such that $\sum_{t=1}^N \indic_{a_{i,t} = a^*(s_{i,t})} \ge n$, then we win the game, {\it i.e.,} the event $\Eps$ is triggered. 
    
    Let us fix $i$ for now and define a few variables:
    \begin{align*}
        X_{s} &:= \sum_{t: s_{i,t} = s} \indic_{a_{i,t} = a^*(s)}, \\
        N_{s} &:= \sum_{t: s_{i,t} = s} 1, \
        N_{s,a} := \sum_{t: s_{i,t} = s} \indic_{a_{i,t} = a},
    \end{align*} 
    for all $s \in \mS, a \in \mA$. Note that $X_s$ is a random variable decided by $a^*(s)$, with $\Exs[X_{s}] = N_s / A$ and $Var(X_s) \le N_s^2 / A$, and $X_s \le N_s$ almost surely. Furthermore, $X_{s}$ and $X_{s'}$ are independent if $s \neq s'$. Thus, now we only need to check the probability of event $\sum_{s\in \mS} X_s \ge n$. Let $t = n - N/A$. Since $\{X_s\}_{s \in \mS}$ are independent random variables, we can apply Bernstein's inequality to obtain:
    \begin{align}
        \PP \left( \sum_{s \in \mS} X_s \ge n \right) &\le \exp \left(- \frac{\frac{1}{2}t^2 }{\sum_{s\in \mS} Var(X_s) + \frac{1}{3} (\max_{s} N_s) t } \right) \nonumber \\
        &\le \exp \left( - \frac{ \frac{1}{2}t^2 }{ \frac{1}{A} \sum_{s\in \mS} N_s^2 + \frac{1}{3} (\max_{s} N_s) t } \right). \label{eq:game_win_event}
    \end{align}
    One thing we note here is that since we assumed uniformly sampled context every step, from Bernstein's inequality, we can bound the maximum of $N_s$ as the following:
    \begin{lemma}
        $N_{max} := \max_{s} N_s \le 2 N/S + O(\log(L/\delta))$ with probability at least $1 - \delta / L^4$.
    \end{lemma}
    This is an application of basic Bernstein's inequality for sums of Bernoulli random variable with parameter $1/S$. Then, using this and $\sum_{s \in \mS} N_s = N$, a simple algebra can show that 
    \begin{align*}
        \sum_{s \in \mS} N_s^2 &\le N \cdot N_{max}.
    \end{align*}
    Under such event, we can bound \eqref{eq:game_win_event} further such that 
    \begin{align*}
        \PP \left( \sum_{s \in \mS} X_s \ge n \right) &\le \exp \left( - \frac{ \frac{1}{2}t^2 }{ N_{\max} (\frac{N}{A}  + \frac{1}{3} t) } \right) \le \exp \left( - \frac{ n }{ N_{\max} } \right),
    \end{align*}
    where we used $n > 10 N/A$ for sufficiently large $A$. Now if $N/S > \log(L/\delta)$, then it is less than $\exp(- nS / N) \le \exp(- c_1 \cdot S^{\beta})$ for some constant $c_1 > 0$. Otherwise, it is less than $\exp(- c_1 \cdot n / \log L)$. In either case, this probability is small enough so that for all users $i \in [L]$, we can take a union bound and show that with no user we have played right actions more than $n$ times if:
    \begin{align*}
        c \cdot S^\beta \gg \log L, \text{ and } c \cdot n \gg \log^2 L,
    \end{align*}
    which holds with our setup $\frac{\log S}{\log (\log L)} > 2 / \beta$ and $n \ge 4\log^2 L$. This concludes the proof for Lemma \ref{lemma:auxiliary_game_A}.
\end{proof}
\vskip 0.2cm

Now suppose that there exists an algorithm that can trigger $\Eps$ in system B with less than $N$ times of per-user interactions with probability more than $2/3$. However, note that system A and B cannot be distinguished before $\Eps$ is triggered due to Lemma \ref{lemma:contamination_tv} and Le Cam \cite{lecam1973convergence} as argued before. However, if it is possible to trigger $\Eps$, {\it i.e.,} to play correct actions at least $n$ times for any user, only using $N$ interactions with probability more than $2/3$, then it is possible to distinguish A and B with probability better than $1/2 \cdot (1 - 1/L^2) + 1/2 \cdot 2/3 \approx 5/6$. Note that until $\Eps$ is triggered, we can only observe reward sequences of length at most $n$ for correct actions. Therefore, this means that we have a hypothesis testing mechanism by detecting the event $\Eps$ only with length $n$ reward sequences for correct actions. This contradicts the fundamental limit of two hypothesis testing.

Equivalently, if there exists an $(\epsilon/3, 1/3)$-PAC algorithm using at most $N/2$ per-user interactions, then we can run this algorithm to obtain an $(\epsilon/3, 1/3)$-PAC policy, and run this policy for the rest of $N/2$ per-user interactions to trigger $\Eps$ in system B. This again contradicts the fundamental limit, and thus we conclude Theorem \ref{theorem:lower_bound}.

\section{Deferred Details in Section \ref{section:main_algorithm}}

\subsection{Recovery Procedure for $\hat{\mu}$ in Algorithm \ref{algo:nonuniform_robust}}
\label{appendix:complete_procedure}
We describe the procedure we deferred in Algorithm \ref{algo:nonuniform_robust} after $T$ time steps. For every context, let $n_i = | \{t| i_t=i \}|$ be the number of interactions with $i$. For top-$A$ frequent contexts $s \in \mS_+$, let $\hat{b}_{i,s} = \frac{1}{n_{i}} \sum_{t: i_t = i} r_t \sigma_{s}^{-1} \indic_{s_t = s}$ be the empirical mean of reward times the context probability from the user $i$ and the context $s$. Then, for every $a \in \mA$, call the univariate robust estimator in Theorem \ref{theorem:univariate_robust_mean} with input $\{\hat{b}_{i,s}\}_{i \in \mI_{s,a}}$, and receive an estimate $\hat{b} (s,a)$. Set $\hat{\mu} (s,a) = \sigma_s^{-1} \cdot \hat{b} (s,a)$ for every $s \in \mS_+$ and $a \in \mA$. 

For those contexts not in $\mS_+$, let $\mS_- = \mS / \mS_+$ and $\hat{b}_i := \frac{A}{n_i} \sum_{t: i_t = i} r_t \sigma_{s_t}^{-1} \indic_{s_t \in \mS_-} \cdot \bm{e}_{(s_t, a_t)}$ where we use $\sigma_s$ found in Algorithm \ref{algo:estimate_nu}. We call the robust estimator in Theorem \ref{theorem:high_dim_robust_estimator} with input $\{\hat{b}_i\}_{i=1}^L$, and get a returned $\hat{b}$. Note that we do not use this estimator for frequent contexts, and thus $\hat{b}(s,a) = 0$ for $s \in \mS_+$. Set $\hat{\mu}(s,a) = \sigma_s^{-1} \cdot \hat{b} (s,a)$ for every $s \in \mS_-$ and $a \in \mA$.

\subsection{Proof of Theorem \ref{theorem:upper_bound}}
We first show that the context probability estimated in Algorithm \ref{algo:estimate_nu} is approximately correct with the following guarantee:
\begin{lemma}
    \label{lemma:estimate_nu}
    Let $\alpha < 1/3$ and $L = \Omega(S \log (S) / \alpha)$. Then Algorithm \ref{algo:estimate_nu} with $T_0 = O(L\alpha / \epsilon^2)$ returns an estimator $\hat{\nu}$ that satisfies
    \begin{align*}
        \| \hat{\nu} - \nu \|_1 \le \epsilon, 
    \end{align*}
    with probability at least $9/10$.
\end{lemma}
\begin{proof}
    We first note that $\hat{\nu}_i := \frac{1}{n_i(T_0)} \sum_{t: i_t=i} \bm{e}_{s_t}$ for each $i \in \mI^*$ satisfies that
    \begin{align*}
        \Exs[\hat{\nu}_i] = \nu, \ Cov(\hat{\nu}_i) = \frac{1}{n_i(T_0)} diag \left(\nu^{-1} \right),
    \end{align*}
    conditioned on the number of interactions $n_i(T_0)$, where $diag \left( \nu^{-1} \right)$ is a diagonalized matrix of vector $\nu^{-1}$ such that $\nu^{-1}(s) = \nu(s)^{-1}$ for all $s \in \mS$. Note that $n_i(T_0) = O(T_0 / L)$. Let a vector $\sqrt{v} \in \mathbb{R}^S$ such that $\sqrt{v} (s) = \sqrt{\nu(s)}$. Then define $\tilde{b}_i := diag \left(\sqrt{v} \right)^{-1} \hat{\nu}_i$ which satisfies
    \begin{align*}
        \Exs[ \tilde{b}_i ] = \sqrt{\nu}, \ Cov(\tilde{b}_i) = \frac{1}{n_i (T_0)} I \preceq O \left(\frac{L}{T_0} I \right).
    \end{align*}
    Since all $\tilde{b}_i$ are independent from each other (conditioned on the order of user interactions decided by external process $\Gamma$), we can find a set of $N$ samples from $\{ \tilde{b}_i \}_{i=1}^L$ that satisfies Condition \ref{condition:deterministic_robust} with $N = (1 - 1.1 \alpha) L$, $\mu^* = \sqrt{\nu}$, $\delta_1 = O \left(\sqrt{\alpha L / T_0} \right)$, $\delta_2 = O(L / T_0)$, and $\delta_3 = O \left(\sqrt{S L / (T_0 \alpha) } \right)$. 
    
    Now we observe that $\hat{b}_i = M \tilde{b}_i$ where $M := diag \left( \{\sigma_s^{-1} \}_{s \in \mS} \right) diag \left(\sqrt{v} \right)$. Once we show that $\|M\|_2 \le O(1)$, then the deterministic condition for robust estimation (Condition \ref{condition:deterministic_robust}) holds with $\mu^*$: $\mu^*(s) = \sigma_s^{-1} \nu(s)$. Once we receive a robust estimate of samples $\{\hat{b}_i\}_{i=1}^L$, by the guarantee given by Theorem \ref{theorem:high_dim_robust_estimator}, we ensure that
    \begin{align*}
        \sum_{s} ( \hat{b}(s) - \sigma_s^{-1} \nu(s) )^2 \le O \left( \sqrt{\alpha L / T_0} \right),
    \end{align*}
    with probability at least $9/10$. Therefore using $\hat{\nu}(s) := \sigma_s \hat{b}(s)$ is guaranteed as the following:
    \begin{align*}
        \sum_{s} | \sigma_s \hat{b}(s) - \nu(s) | &= \sum_{s} \sigma_s \cdot | \hat{b}(s) - \sigma_s^{-1} \nu(s) | \\ 
        &\le \sqrt{\sum_{s} \sigma_s^{2} } \sqrt{ \sum_{s} ( \hat{b}(s) - \sigma_s^{-1} \nu(s) )^2 } \\
        &\le \sqrt{\frac{\sum_{s} n_s(T_0) + O(\log S)}{T_0}} \cdot O \left(\sqrt{\alpha T_0 / L} \right) \\
        &= \sqrt{1 + O(S\log S)/T_0} \cdot O \left(\sqrt{\alpha L / T_0 } \right).
    \end{align*}
    Since $T_0 = O(L \alpha / \epsilon^2) \ge S \log(S) / \epsilon^2$, the right-hand side is less than $\epsilon$.
    
    Finally, we show that for all $s \in \mS$, 
    \begin{align}
        \max(20 \log(S) / T_0, 2v(s)) \ge \sigma_s^{2} \ge \max( \log(S) / T_0, v(s) / 4). \label{eq:prob_scale}
    \end{align}
    If $v(s) < 4 \log (10 S) / T_0$, then this is true by the definition of $\sigma_s^{2} = \max( n_s(T_0), 20 \log S) / T_0$. Otherwise, we can show it by a straight-forward application of Bernstein's inequality: let $T_0' = | \{t| i_t \in \mI^* \} |$ be the number of times the system interacts with good users. With probability at least 0.99, $T_0' \ge (1 - 1.1 \alpha) T_0$. Then, 
    \begin{align*}
        \PP \left( \left| \sum_{t: i_t \in \mI^*} \indic_{s_t = s} - T_0' \cdot v(s) \right| > t \right) &\le \exp \left(-\frac{ \frac{1}{2} t^2}{\sum_{t: i_t \in \mI^*} \Exs[\indic_{s_t = s}^2 ] + \frac{1}{3}t} \right) \\
        &\le \exp \left(-\frac{ \frac{1}{2} t^2}{T_0' \cdot v(s) + \frac{1}{3}t} \right).
    \end{align*}
    Plugging $t = T_0' \cdot v(s) / 2$ and $v(s) > 4 \log (10 S) / T_0$, we can conclude that 
    \begin{align*}
        n_s(T_0) \ge \sum_{t: i_t \in \mI^*} \indic_{s_t = s} \ge T_0' \cdot v(s) / 2 > T_0 \cdot v(s) / 4,
    \end{align*}
    with probability at least $1 - 1/ (100S^3)$. Taking union bound over all $s$, we have $\sigma_s^{2} \ge v(s) / 4$ with probability at least $0.99$, yielding $\|M\|_2 \le 2 = O(1)$. This concludes Lemma \ref{lemma:estimate_nu}. 
\end{proof}
\vskip 0.2cm

\paragraph{Guarantees for Univariate Estimators.} Rest of the proof follows the similar logic. We first show that for each estimator $\hat{\mu}(s,a)$ for $s \in \mS_+$ and $a \in \mA$, it holds that 
\begin{align*}
    | \sigma_s^2 \hat{\mu}(s,a) - \nu(s) \mu(s,a) | \le O \left( \sqrt{\nu(s) \alpha T / L} \right). 
\end{align*}
To show this, we only need to see that $\hat{b}_{i,s}$ satisfies
\begin{align*}
    \Exs[\hat{b}_{i,s}] = \sigma_s^{-1} \nu (s) \mu(s, a_{i,s}), \ Cov(\hat{b}_{i,s}) \preceq \frac{\sigma_s^{-2} \nu (s)}{n_i} I.
\end{align*}
Therefore, after running an univariate robust estimator (Theorem \ref{theorem:univariate_robust_mean}) with input $\{\hat{b}_{i,s}\}_{i \in \mI_{s,a}}$ for each $s \in \mS_+, a \in \mA$, and with Assumption \ref{assumption:user_selection} so that $n_i = O(L/T)$, we get
\begin{align*}
    | \hat{b}_{s,a} - \sigma_s^{-1} \nu(s) \mu(s, a_{i,s}) | \le O\left(  \sigma_s^{-1} \sqrt{\nu(s) \alpha L / T} \right) = O\left( \sqrt{\alpha L / T} \right),
\end{align*}
with probability at least $1 - 0.1 / (SA)$ given $|\mI_{s,a}| = \Omega(\log(SA) / \alpha)$. We used the fact from \eqref{eq:prob_scale}. 

To compute the error contributed from $\mS_+$ part, we can observe that
\begin{align*}
    \sum_{s \in \mS_+} \nu (s) \cdot \left( \mu(s, \pi^*(s)) - \mu(s, \hat{\pi}(s)) \right) &\le \sum_{s \in \mS_+} \nu(s) \mu(s, \pi^*(s)) - \nu(s) \mu(s, \hat{\pi}(s)) + \sigma_s^2 \hat{\mu}(s,\hat{\pi}(s)) - \sigma_s^2 \hat{\mu}(s,\pi^*(s)) \\
    &\le \sum_{s \in \mS_+} |\nu(s) \mu(s, \pi^*(s)) - \sigma_s^2 \hat{\mu}(s, \pi^*(s) | + | \nu(s) \mu(s, \hat{\pi}(s)) - \sigma_s^2 \hat{\mu}(s,\hat{\pi}(s)) | \\
    &\le \left( \sum_{s \in \mS_+} \sigma_s \right) \cdot \max_{s \in \mS_+} | \hat{b}_{s,\pi^*(s)} - \sigma_s^{-1} \nu(s) \mu(s, \pi^*(s)) | \\
    &\quad + \left( \sum_{s \in \mS_+} \sigma_s \right) \cdot \max_{s \in \mS_+} | \hat{b}_{s,\hat{\pi}(s)} - \sigma_s^{-1} \nu(s) \mu(s, \hat{\pi}(s)) | \\
    &\le \left( \sum_{s \in \mS_+} \sigma_s \right) \cdot O\left( \sqrt{\alpha L / T} \right).
\end{align*}
From \eqref{eq:prob_scale}, we have $\sigma_s^2 \le O(\log S / T_0) + 2v(s)$, where $T_0 = O(\alpha L / \epsilon^2) = O(SA \log(SA) / \epsilon^2)$. Hence, $\sigma_s^2 \le O(\epsilon^2 / SA) + 2 \nu(s)$, which gives
\begin{align*}
    \sum_{s \in \mS_+} \sigma_s &\le \sum_{s \in \mS_+} O\left( \epsilon / \sqrt{SA} + \sqrt{\nu(s)} \right) \le O \left(\epsilon + \sum_{s \in \mS_+} \sqrt{\nu(s)} \right),
\end{align*}
where we used $|\mS_+| \le \min(S,A) \le \sqrt{SA}$. Note that by definition of $s_{[j]}$ for $j = 1, 2, ..., S$,
\begin{align*}
    \sum_{s \in \mS_+} \sqrt{\nu(s)} \le \sum_{i = 1}^{\min(S,A)} \sqrt{\nu(s_{[i]})}.
\end{align*}

\paragraph{Guarantees for High-Dimensional Estimators.} Recall that $\hat{b}_i := \frac{A}{n_i} \sum_{t: i_t = i} r_t \sigma_{s_t}^{-1} \indic_{s_t \in \mS_-} \cdot \bm{e}_{(s_t, a_t)}$. Since we only care about $s \in \mS_-$, we restrict ourselves to coordinates in $\mS_-$. For a vector $v$ and a index-set $\mI$, we denote $v_{\mI}$ as a restriction of a vector to coordinates only in $\mI$. $b_{i,\mS_-}$. Similarly to the uni-variate case, we first see the expectation and covariance of $\hat{b}_i$:
\begin{align*}
    \Exs[\hat{b}_{i, \mS_-}] = diag( \{\sigma_s^{-1}\}_{s \in \mS_-} ) diag(\nu_{\mS_-}) \mu_{\mS_-}, \ Cov(\hat{b}_{i,\mS_-}) \preceq diag( \{\sigma_s^{-2}\}_{s \in \mS_-} ) diag(\nu_{\mS_-}) \frac{LA}{T} I.
\end{align*}
From this, the high-dimensional robust estimator in Theorem \ref{theorem:high_dim_robust_estimator} is guaranteed to return $\hat{b}_{\mS_-}$ such that
\begin{align*}
    \|\hat{b}_{\mS_-} - diag( \{\sigma_s^{-1}\}_{s \in \mS_-} ) diag(\nu_{\mS_-}) \mu_{\mS_-} \|_2 \le O \left(\sqrt{\alpha AL / T} \right). 
\end{align*}
We used \eqref{eq:prob_scale} to bound $\|diag( \{\sigma_s^{-2}\}_{s \in \mS_-} ) diag(\nu_{\mS_-})\|_2 = O(1)$. From this, we can bound the errors from less frequent contexts $\mS_-$. We first note that
\begin{align*}
    \sum_{s \in \mS_-} \hat{\nu} (s) &\le \sum_{j=A+1}^S \hat{\nu} (s_{[j]}) \le \sum_{j=A+1}^S \nu (s_{[j]}) + \epsilon,
\end{align*}
where the first inequality comes from the fact that $\mS_-$ is a collection of contexts that does not belong to top-$A$ highest probabilities in $\hat{\nu}$. Also, by Lemma \ref{lemma:estimate_nu}, we have 
\begin{align}
    \sum_{s \in \mS_-} \nu(s) \le \sum_{s \in \mS_-} \hat{\nu}(s) + \epsilon \le \sum_{j=A+1}^S \nu (s_{[j]}) + 2\epsilon. \label{eq:estiamte_nu_order}
\end{align}
Having this, we can show that
\begin{align*}
    \sum_{s \in \mS_-} \nu(s) \cdot \left( \mu(s, \pi^*(s)) - \mu(s, \hat{\pi}(s)) \right) &\le \sum_{s \in \mS_-} |\nu(s) \mu(s, \pi^*(s)) - \sigma_s^2 \hat{\mu}(s, \pi^*(s) | + | \nu(s) \mu(s, \hat{\pi}(s)) - \sigma_s^2 \hat{\mu}(s,\hat{\pi}(s)) | \\
    &\le \sqrt{ \sum_{s \in \mS_-} \sigma_s^2 } \cdot \sqrt{ \sum_{s \in \mS_-} \left(\hat{b}_{s,\pi^*(s)} - \sigma_s^{-1} \nu(s) \mu(s, \pi^*(s)) \right)^2 } \\
    &\quad + \sqrt{ \sum_{s \in \mS_-} \sigma_s^2 } \cdot \sqrt{ \sum_{s \in \mS_-} \left(\hat{b}_{s, \hat{\pi} (s)} - \sigma_s^{-1} \nu(s) \mu(s, \hat{\pi}(s)) \right)^2 } \\
    &\le \sqrt{ \sum_{s \in \mS_-} \sigma_s^2 } \cdot O\left( \sqrt{\alpha L A / T} \right).
\end{align*}
In order to bound $\sum_{s \in \mS_-} \sigma_s^2$, we use \eqref{eq:prob_scale} and see that
\begin{align*}
    \sum_{s \in \mS_-} \sigma_s^2 \le 2 \sum_{s \in \mS_-} \nu(s) + 20 \frac{S \log(S)}{T_0} = O \left( \sum_{s \in \mS_-} \nu(s) + \epsilon^2 / A \right),
\end{align*}
where we use $T_0 = O(L \alpha / \epsilon^2) \ge SA \log(SA) / \epsilon^2$. Finally, we plug \eqref{eq:estiamte_nu_order}, and we have
\begin{align*}
    \sum_{s \in \mS_-} \nu(s) \cdot \left( \mu(s, \pi^*(s)) - \mu(s, \hat{\pi}(s)) \right) \le O \left( \epsilon + \sqrt{\sum_{j=A+1}^S \nu(s_{[j]}) \cdot A} \right) \sqrt{\frac{\alpha L}{T}}. 
\end{align*}
Combining this result with the bound for contexts in $\mS_+$, we get Theorem \ref{theorem:upper_bound}. 
\end{appendices}

\end{document}